\documentclass{article} 
\usepackage{hyperref}
\usepackage{url}

\usepackage{times}
\usepackage{graphicx} 
\usepackage{cite}
\usepackage{algorithm}
\usepackage{algorithmicx}
\usepackage{algpseudocode}
\usepackage{amsmath}
\usepackage{amsfonts}
\usepackage{bbm}
\usepackage{titlesec}
\usepackage{caption}
\usepackage{subcaption}
\usepackage{amsthm}
\usepackage[affil-it]{authblk}

\newtheorem{theorem}{Theorem}
\newtheorem{lemma}{Lemma}

\DeclareMathOperator{\sgn}{sgn}
\DeclareMathOperator{\argmin}{argmin}

\title{Efficient Elastic Net Regularization \\ for Sparse Linear Models}

\author{Zachary C. Lipton \qquad \qquad Charles Elkan \\
University of California, San Diego\\
\texttt{\{zlipton, elkan\}@cs.ucsd.edu}
}

\date{May 27th, 2015}

\begin{document}

\maketitle
\begin{abstract}
This paper presents an algorithm for efficient training of sparse linear models with elastic net regularization. 
Extending previous work on delayed updates,
the new algorithm applies stochastic gradient updates to non-zero features only,
bringing weights current as needed with closed-form updates.
Closed-form delayed updates for the $\ell_1$, $\ell_{\infty}$,
and rarely used $\ell_2$ regularizers have been described previously.
This paper provides closed-form updates for the popular squared norm $\ell^2_2$ and elastic net regularizers.
We provide dynamic programming algorithms that  perform each delayed update in constant time.
The new $\ell^2_2$ and elastic net methods handle both fixed and varying learning rates,
and both standard {stochastic gradient descent} (SGD) 
and {forward backward splitting (FoBoS)}.
Experimental results show that on a bag-of-words dataset with $260,941$ features,
but only $88$ nonzero features on average per training example,
the dynamic programming method trains a logistic regression classifier with elastic net regularization
over $2000$ times faster than otherwise.
\end{abstract}

\section{Introduction}
\label{section:introduction}
For many applications of linear classification or linear regression,
training and test examples are sparse,
and with appropriate regularization, 
a final trained model can be sparse and still achieve high accuracy.
It is therefore desirable to be able to train linear models
using algorithms that require time that scales only with the number of non-zero feature values.

Incremental training algorithms such as stochastic gradient descent (SGD)
are widely used to learn high-dimensional models from large-scale data.
These methods process each example one at a time or in small batches,
updating the model on the fly.
When a dataset is sparse and the loss function is not regularized,
this sparsity can be exploited
by updating only the weights corresponding to non-zero feature values for each example.
However, to prevent overfitting to high-dimensional data,
it is often useful to apply a regularization penalty,
and specifically to impose a prior belief that the true model parameters
are sparse and small in magnitude.
Unfortunately, widely used regularizers such as $\ell_1$ (lasso), $\ell_2^2$ (ridge),
and elastic net ($\ell_1 + \ell_2^2)$
destroy the sparsity of the stochastic gradient for each example,
so they seemingly require most weights to be updated for every example.

This paper builds upon methods for delayed updating of weights
first described by \cite{carpenter2008lazy}, \cite{singer2009efficient}, and \cite{langford2009sparse}.
As each training example is processed,
the algorithm updates only those weights corresponding to non-zero feature values in the example.
The model is brought current as needed 
by first applying closed-form constant-time delayed updates for each of these weights.
For sparse data sets, the algorithm runs in time independent of the nominal dimensionality, 
scaling linearly with the number of non-zero feature values per example.
 
To date, constant-time delayed update formulas have been derived
only for the $\ell_1$ and $\ell_{\infty}$ regularizers, and for the rarely used $\ell_2$ regularizer.
We extend previous work by showing the proper closed form updates
for the popular $\ell^2_2$ squared norm regularizer and for elastic net regularization.
When the learning rate varies (typically decreasing as a function of time),
we show that the elastic net update can be computed
with a dynamic programming algorithm that requires only constant-time computation per update.%
\footnote{The dynamic programming algorithms below
for delayed updates with varying learning rates use time O(1) per update,
but have space complexity $O(T)$
where $T$ is the total number of stochastic gradient updates.
If this space complexity is too great, that problem can be solved by allotting a fixed space budget 
and bringing all weights current whenever the budget is exhausted. 
As the cost of bringing weights current is amortized across many updates, 
it adds negligibly to the total running time.
}

A straightforward experimental implementation of the proposed methods
shows that on a representative dataset containing
the abstracts of a million articles from biomedical literature,
we can train a logistic regression classifier with elastic net regularization
over $2000$ times faster than using an otherwise identical implementation
that does not take advantage of sparsity.
Even if the standard implementation exploits sparsity when making predictions during training,
additionally exploiting sparsity when doing updates, via dynamic programming,
still makes learning $1400$ times faster.
%
%

\section{Background and Definitions}
\label{section:background}

We consider a data matrix $X \in \mathbbm{R^{n \times d}}$ 
where each row $\boldsymbol{x_i}$ is one of $n$ examples
and each column, indexed by $j$, corresponds to one of $d$ features.
We desire a linear model parametrized by a weight vector $\boldsymbol{w} \in \mathbbm{R^d}$
that minimizes a convex objective function $F(\boldsymbol{w})$
expressible as $\sum_{i=1}^{n} F_i (\boldsymbol{w})$,
where $F$ is the loss with respect to the entire dataset $X$
and $F_i$ is the loss due to example $\boldsymbol{x}_i$.

In many datasets, the vast majority of feature values $x_{ij}$ are zero.
The-bag-of-words representation of text is one such case.
We say that such datasets are {sparse}.
When features correspond to counts or to binary values, as in bag-of-words,
we sometimes say say that a zero-valued entry $x_{ij}$ is {absent}.
We use $p$ to refer to the average number of nonzero features per example.
Naturally, when a dataset is sparse,
we prefer algorithms that take time $O(p)$ per example to those that require time $O(d)$.

\subsection{Regularization}
\label{section:background-regularization}

To prevent overfitting, regularization restricts the freedom of a model's parameters,
penalizing their distance from some prior belief.
Widely used regularizers penalize large weights with an objective function of the form
\begin{equation}
\label{eqn:objective-function}
F(\boldsymbol{w}) = {L}(\boldsymbol{w}) + R(\boldsymbol{w}).
\end{equation}
Many commonly used regularizers $R(\boldsymbol{w})$
are of the form $\lambda||\boldsymbol{w}||$
where $\lambda$ determines the strength of regularization and
the  $\ell_0$, $\ell_1$, $\ell^2_2$, or $\ell_{\infty}$ norms
are common choices for the penalty function.
The $\ell_1$ regularizer is popular
owing to its tendency to produce sparse models.
In this paper, we focus on elastic net, 
a linear combination of $\ell_1$ and $\ell^2_2$ regularization
that has been shown to produce comparably sparse models
to $\ell_1$ while often achieving superior accuracy \cite{zou2005regularization}.

\subsection{Stochastic Gradient Descent}
\label{section:background-optimization}

Gradient descent is a common strategy to learn optimal parameter values $\boldsymbol{w}$.
To minimize $F(\boldsymbol{w})$, a number of steps $T$, indexed by $t$, are taken
in the direction of the negative gradient:
$$
\boldsymbol{w}^{(t+1)} := \boldsymbol{w}^{(t)} - \eta \sum_{i=1}^{n} \nabla F_i(\boldsymbol{w})
$$
where the learning rate $\eta$ may be a function of time $t$.
An appropriately decreasing $\eta$ ensures 
that the algorithm will converge to a vector $\boldsymbol{w}$ within distance $\epsilon$
of the optimal vector for any small value~$\epsilon$~\cite{carpenter2008lazy}.

Traditional (or ``batch") gradient descent requires a pass through the entire dataset for each update.
{Stochastic gradient descent} (SGD) circumvents this problem
by updating the model once after visiting each example.
With SGD, 
examples are randomly selected one at a time or in small so-called mini-batches.
For simplicity of notation, without loss of generality we will assume that examples are selected one at a time.
At time $t+1$ the gradient $\nabla F_i(\boldsymbol{w}^{(t)})$ is calculated 
with respect to the selected example $\boldsymbol{x_i}$,
and then the model is updated according to the rule 
$$
\boldsymbol{w}^{(t+1)} := \boldsymbol{w}^{(t)} - \eta \nabla F_i(\boldsymbol{w}^{(t)}).
$$
Because the examples are chosen randomly,
the expected value of this noisy gradient
is identical to the true value of the gradient
taken with respect to the entire corpus.

Given a continuously differentiable convex objective function $F(\boldsymbol{w})$, 
stochastic gradient descent is known to converge for learning rates $\eta$
that satisfy $\sum_t \eta_t = \infty$ and $\sum_t \eta_t^2 < \infty$ \cite{bottou2012stochastic}.
Learning rates $\eta_i \propto {1}/{t}$ and $\eta_i \propto {1}/{\sqrt {t}}$ both satisfy these properties.%
\footnote{
Some common objective functions, such as those involving $\ell_1$ regularization,
are not differentiable when weights are equal to zero.
However, forward backward splitting (FoBoS)
offers a principled approach to this problem \cite{singer2009efficient}.
}
For many objective functions, 
such as those of linear or logistic regression without regularization, 
the noisy gradient $\nabla F_i(\boldsymbol{w})$ is sparse when the input is sparse. 
In these cases, one needs only to update the weights corresponding to non-zero features 
in the current example $\boldsymbol{x}_i$. 
These updates require time $O(p)$, 
where $p \ll d$ is the average number of nonzero features in an example. 

Regularization, however, can ruin the sparsity of the gradient.
Consider an objective function as in Equation~(\ref{eqn:objective-function}),
where $R(w) = ||\boldsymbol{w}||$ for some norm $|| \cdot ||$.
In these cases, even when the feature value $x_{ij} = 0$,
the partial derivative $(\partial / \partial w_j) F_i$ is nonzero owing to the regularization penalty if $w_j$ is nonzero.
A simple optimization is to update a weight only when either the weight or feature is nonzero.
Given feature sparsity and persistent model sparsity throughout training,
not updating $w_j$ when $w_j = 0$ and $x_{ij} = 0$ provides a substantial benefit.
But such an approach still scales with the size of the model, which may be far larger than $p$.
In contrast, the algorithms below scale in time complexity $O(p)$.

%
%
\subsection{Forward Backward Splitting}
\label{section:background-fobos}

Proximal algorithms are an approach to optimization in which each update consists 
of solving a convex optimization problem \cite{parikh2013proximal}.
Forward Backward Splitting (FoBoS) \cite{singer2009efficient}
is a proximal algorithm that provides a principled approach
to online optimization with non-smooth regularizers.
We first step in the direction of the negative gradient of the differentiable unregularized loss function.
We then update the weights by solving a convex optimization problem
that simultaneously penalizes distance from the new parameters
and minimizes the regularization term.

In FoBoS, first a standard unregularized stochastic gradient step is applied:
\begin{equation*}
\boldsymbol{w}^{(t+\frac{1}{2})} = \boldsymbol{w}^{(t)} - \eta^{(t)} \nabla {L}_i(\boldsymbol{w}^{(t)}).
\end{equation*}
Note that if 
$(\partial / \partial w_j) L_i = 0$
then $w_j^{(t+\frac{1}{2})} = w_j^{(t)}$. 
Then a convex optimization is solved, applying the regularization penalty. 
For elastic net the problem to be solved is
\begin{equation}
\boldsymbol{w}^{(t+1)} = \argmin_w \left( \frac{1}{2} || \boldsymbol{w}-\boldsymbol{w}^{(t+\frac{1}{2})} ||_2^2 
	+ \eta^{t}\lambda_1 ||\boldsymbol{w}||_1 + \frac{1}{2} \eta^{t} \lambda_2 ||\boldsymbol{w}||_2^2 \right).
\label{eqn:fobos-update}
\end{equation}
The problems corresponding to $\ell_1$ or $\ell_2^2$ separately can be derived by setting the corresponding $\lambda$ to~$0$.

%
%
\section{Lazy Updates}
\label{section:lazy-updates}

The idea of lazy updating was introduced in \cite{carpenter2008lazy}, \cite{langford2009sparse}, and \cite{singer2009efficient}.
This paper extends the idea for the cases of $\ell_2^2$ and elastic net regularization.
The essence of the approach is given in Algorithm~(\ref{alg:lazy}). 
We maintain an array $\boldsymbol{\psi} \in \mathbbm{R}^d$ 
in which each $\psi_j$ stores the index of the last iteration at which the value of weight $j$ was current.
When processing example $\boldsymbol{x}_i$ at time $k$,
we iterate through its nonzero features $x_{ij}$.
For each such nonzero feature, 
we lazily apply the $k-\psi_j$ delayed updates collectively in constant time,
bringing its weight $w_j$ current.
Using the updated weights, we compute the prediction $\hat{y}^{(k)}$
with the fully updated relevant parameters from $\boldsymbol{w}^{(k)}$.
We then compute the gradient and update these parameters.

When training is complete,
we pass once over all nonzero weights
to apply the delayed updates to bring the model current.
Provided that we can apply any number $k-\psi_j$ of delayed updates in $O(1)$ time,
the algorithm processes each example
in $O(p)$ time regardless of the dimension $d$.

\begin{algorithm}[t]
\caption{Lazy Updates}\label{}
\begin{algorithmic}
\Require{$\boldsymbol{\psi} \in \mathbbm{R}^d$}
\For{$t \in 1,...,T$}
	\State{Sample $x_i$ randomly from the training set}
	\For{$j$  s.t. $x_{ij} \neq 0$}
		\State{$w_j \gets Lazy(w_j, t, \psi_j)$}
		\State{$\psi_{j} \gets t$}
	\EndFor
	\State{$\boldsymbol{w} \gets \boldsymbol{w} - \nabla F_i(\boldsymbol{w})$}
\EndFor
\end{algorithmic}
\label{alg:lazy}
\end{algorithm}

To use the approach with a chosen regularizer,
it remains only to demonstrate the existence of constant time updates.
In the following subsections,
we derive constant-time updates for $\ell_1$, $\ell_2^2$ and elastic net regularization,
starting with the simple case where the learning rate $\eta$ is fixed during each epoch,
and extending to the more complicated case
when the learning rate is decreased every iteration as a function of time.%
\footnote{
The results hold for schedules of weight decrease that depend on time, 
but cannot be directly applied to AdaGrad  \cite{duchi2011adaptive} or RMSprop,
methods where each weight has its own learning rate 
which is decreased with the inverse of the accumulated sum (or moving average)
of squared gradients with respect to that weight.
}

%
%
\section{Prior Work}
\label{section:prior-work}
Over the last several years,
a large body of work has advanced the field of online learning.
Notable contributions include ways of adaptively decreasing the learning rate
separately for each parameter such as AdaGrad \cite{duchi2011adaptive} and AdaDelta \cite{zeiler2012adadelta},
using small batches to reduce the variance of the noisy gradient \cite{li2014efficient},
and other variance reduction methods such as Stochastic Average Gradient (SAG) \cite{schmidt2013minimizing}
and Stochastic Variance Reduced Gradient (SVRG) \cite{johnson2013accelerating}.

In 2008, Carpenter described an idea
for performing lazy updates for stochastic gradient descent  \cite{carpenter2008lazy}.
With that method, we maintain a vector $\boldsymbol{\psi} \in \mathbbm{N}^d$,
where each $\psi_i$ stores the index of the last epoch
in which each weight was last regularized.
We then perform periodic batch updates.
However, as the paper acknowledges,
the approach described results in updates
that do not produce the same result
as applying an update after each time step.

Langford \emph{et al.} concurrently developed an approach for
lazily updating $\ell_1$ regularized linear models \cite{langford2009sparse}.
They restrict attention to $\ell_1$ models.
Additionally, they describe the closed form update only
when the learning rate $\eta$ is constant,
although they suggest that an update can be derived when $\eta_t$ decays as $t$ grows large.
We derive constant-time updates
for $\ell_2^2$ and elastic net regularization.
Our algorithms are applicable with both fixed and varying learning rates.

In 2008 also, as mentioned above, Duchi and Singer described the FoBoS method \cite{duchi2008efficient}.
They share the insight of applying updates lazily
when training on sparse high-dimensional data.
Their lazy updates hold for norms $\ell_q$ for $q \in \{1,2,\infty \}$,
However they do not hold for the commonly used $\ell^2_2$ squared norm.
Consequently they also do not hold for mixed regularizers
involving $\ell^2_2$ such as the widely used elastic net $(\ell_1 + \ell^2_2)$.

%
%
\section{Constant-Time Lazy Updating for SGD}
\label{section:lazy-attenuated}

In this section, we derive constant-time stochastic gradient updates
for use when processing examples from sparse datasets.
Using these, the lazy update algorithm can train linear models
with time complexity $O(p)$ per example.
For brevity, we describe the more general case where the learning rate is varied.
When the learning rate is constant the algorithm can be easily modified to have $O(1)$ space complexity.

%
%
\subsection{Lazy $\ell_1$ Regularization with Decreasing Learning Rate}
\label{section:lazy-lasso-attenuated}

The closed-form update for $\ell_1$ regularized models is \cite{singer2009efficient}
\begin{equation*}
w_j^{(k)} = \sgn(w_j^{(\psi_j)}) \left[ |w_j^{(\psi_j)}| -   \lambda_1  \left( S(k-1) - S(\psi_j -1) \right) \right]_+
\label{eq:l1_lazy_update}
\end{equation*}
where $S(t)$ is a function that returns
the partial sum $ \sum_{\tau=0}^{t} \eta^{(\tau)}$.
The sum $\sum_{\tau=t}^{t+n-1} \eta^{(\tau)}$
can be computed in constant time using a caching approach.
On each iteration $t$, we compute $S(t)$ in constant time
given its predecessor as $S(t) = \eta^{(t)} + S(t-1)$.
The base case for this recursion is $S(0) = \eta^{(0)}$.
We then cache this value in an array for subsequent constant time lookup.

When the learning rate decays with  $1/t$,
the terms $\eta ^ {(\tau)}$ follow the harmonic series,.
Each partial sum of the harmonic series
is a harmonic number $H(t) = \sum_{i=1}^{t} {1}/{t}$.
Clearly
$$
\sum_{\tau =t}^{t+n-1} \eta^{(\tau)} = \eta^{(0)} \left( H(t+n) - H(t) \right)
$$
where $H_\tau$ is the $\tau_{th}$ harmonic number.
While there is no closed-form expression
to calculate the $\tau_{th}$ harmonic number,
there exist good approximations.

The $O(T)$ space complexity of this algorithm may seem problematic.
However, this problem is easily dealt with by bringing all weights current after each epoch.
The cost to do so is amortized across all iterations and is thus negligible.

%
%
\subsection{Lazy $\ell_2^2$ Regularization with Decreasing Learning Rate}
\label{section:lazy-ridge-attenuated}

For a given example $\boldsymbol{x}_i$,
if the feature value $x_{ij} = 0$ and the learning rate is varying,
then the stochastic gradient update rule for an $\ell^2_2$ regularized objective is
\begin{equation*} \label{}
w_j^{(t+1)} = w_j^{(t)} - \eta^{(t)} \lambda_2 w_j^{(t)}.
\end{equation*}
The decreasing learning rate prevents collecting successive updates as terms in a geometric series,
as we could if the learning rate were fixed.
However, we can employ a dynamic programming strategy.

\begin{lemma}
For SGD with $\ell_2^2$ regularization,
the constant-time lazy update to bring a weight current from iteration
$\psi_j$ to $k$ is
\begin{equation*}
	w_j^{(k)} = w_j^{(\psi_j)} \frac{P(k-1)}{P(\psi_j -1)}
\end{equation*}
where P(t) is the partial product $\prod_{\tau=0}^{t} (1- \eta^{(\tau)} \lambda_2).$
\end{lemma}

\begin{proof}
Rewriting the multiple update expression yields
\begin{equation*} \label{}
\begin{split}
	w_j^{(t+1)} 	& = w_j^{(t)} ( 1 - \eta^{(t)} \lambda_2 ) \\
	w_j^{(t+n)}	& = w_j^{(t)} ( 1 - \eta^{(t)} \lambda_2 )( 1 - \eta^{(t+1)} \lambda_2 )
					\cdot ... \cdot
					(1 - \eta^{(t+n-1)} \lambda_2).
\end{split}
\end{equation*}
The products $P(t) = \prod_{\tau=0}^{t}  (1 - \eta^{(\tau)} \lambda_2) \ $
can be cached on each iteration in constant time using the recursive relation
$$
P(t)=(1-\eta^{(t)}\lambda_2)P(t-1).
$$
The base case is $P(0) = a_0 = (1-\eta_0 \lambda_2).$
Given cached values $P(0),...,P(t+n)$, it is then easy
to calculate the exact lazy update in constant time:
\begin{equation*}
	w_j^{(t+n)} = w_j^{(t)} \frac{P(t+n-1)}{P(t-1)}.
\end{equation*}
The claim follows.
\end{proof}

As in the case of $\ell_2^2$ regularization with fixed learning rate,
we need not worry that the regularization update
will flip the sign of the weight $w_j$,
because $P(t) >0$ for all $t \geq 0$.

%
%
\subsection{Lazy Elastic Net Regularization with Decreasing Learning Rate}
\label{section:lazy-elastic-attenuated}

Next, we derive the constant time lazy update for SGD with elastic net regularization.
Recall that a model regularized by elastic net has an objective function of the form
$$
F(w) = {L}(w) + \lambda_1 ||w ||_1 + \frac{\lambda_2}{2} ||w||_2^2.
$$
When a feature $x_j = 0$, the SGD update rule is
\begin{equation}
\label{eqn:elastic-attenuated-update}
\begin{split}
w_j^{(t+1)}	& = \sgn(w_j^{(t)}) \left[| w_j^{(t)}| -
				\eta^{(t)} \lambda_1 -\eta^{(t)} \lambda_2 |w_j^{(t)}| \right]_+ \\
			& =  \sgn(w_j^{(t)}) \left[(1 - \eta^{(t)} \lambda_2)| w_j^{(t)}| -
				\eta^{(t)} \lambda_1 \right]_+
\end{split}
\end{equation}

\begin{theorem}
To bring the weight $w_j$ current from time $\psi_j$ to time $k$
using repeated Equation~(\ref{eqn:elastic-attenuated-update}) updates, the constant time update is
\begin{equation*}
w_j^{(k)} = \sgn(w_j^{(\psi_j)}) \left[ |w_j^{(\psi_j)}| \frac{P(k-1)}{P(\psi_j -1)}  -
			 P(k-1) \cdot \left( B(k-1) - B(\psi_j -1) \right) \right]_+
\end{equation*}
where $P(t) = (1-\eta^{(t)} \lambda_2) \cdot P(t-1)$ with base case $P(-1) = 1$ and 
$B(t) = \sum_{\tau = 0}^{t} {\eta^{(\tau)}}/{P(\tau - 1) }$ with base case $B(-1) = 0$.
\end{theorem}

\begin{proof}

The time-varying learning rate prevents us from working out a simple expansion. 
Instead, we can write the following inductive expression for consecutive terms in the sequence:
\begin{equation*}
\begin{split}
	w_j^{(t+1)} & =\sgn(w_j^{(t)}) \left[ (1- \eta^{(t)} \lambda_2) |w_j^{(t)}| - \eta^{(t)} \lambda_1   \right]_+ \\
\end{split}
\end{equation*}
Writing $a_{\tau} = (1- \eta^{(\tau)} \lambda_2)$ and $b_{\tau} = -\eta^{(\tau)} \lambda_1$ gives
\begin{equation*}
\begin{split}
	w^{(t+1)} & = \sgn(w_j^{(t)}) \left[ a_t |w^{(t)}| + b_t \right]_+ \\
			& ... \\
	w^{(t+n)} 	& = \sgn(w_j^{(t)}) \left[
		a_{(t+n-1)} ( ... a_{(t+1)} \left( a_{t} w^{t} - b_t \right) - b_{(t+1)} ... ) - b_{(t+n-1)}
		\right]_+ \\
	& = \sgn(w_j^{(t)}) \left[ |w_j^{(t)}| \prod_{\tau=t}^{t+n-1} a_{\tau}
		+ \sum_{\tau= t}^{t+n-2} b_i \left( \prod_{q=\tau}^{t+n-2} a_q \right) + b_{(t+n-1)}
		\right]_+
\end{split}
\end{equation*}
The leftmost term $\prod_{\tau=t}^{t+n-1} a_{\tau}$ can be calculated
in constant time as $P(t+n-1)/P(t-1)$
using cached values from the dynamic programming scheme
discussed in the previous section.
To cache the remaining terms, we group the center
and rightmost terms and apply the simplification
$$
\sum_{\tau= t}^{t+n-2} b_i \left( \prod_{q=\tau}^{t+n-2} a_q \right) + b_{t+n-1}
$$
\begin{equation*}
\begin{split}
	& = b_{t} \frac{P(t+n-2)}{P(t-1)} + b_{t+1} \frac{P(t+n-2)}{P(t)} + ... + b_{t+n-1} \frac{P(t+n-2)}{P(t+n-2)} \\
	& = 	- \lambda_1 P(t+n-2) \left( \frac{\eta^{(t)}}{P(t-1) }
			+  \frac{\eta^{(t+1)}}{P(t)} +  ... + \frac{\eta^{(t+n-1)}}{P(t+n-2)} \ \ \right).
\end{split}
\end{equation*}

We now add a new layer to the dynamic programming formulation.
In addition to precalculating all values $P(t)$ as we go,
we define a partial sum over inverses of partial products
$$
B(t) = \sum_{\tau = 0}^{t} \frac{\eta^{(\tau)}}{P(\tau - 1) }.
$$
Given that $P(t-1)$ can  be accessed in constant time at time $t$,
$B(t)$ can now be cached in constant time.
With the base case $B(-1) = 0$, the dynamic programming here depends upon the recurrence relation
$$B(t) = B(t-1) + \frac{\eta^{(t)}}{P(t - 1) } .$$

Then, for SGD elastic net with decreasing learning rate,
the update rule to apply any number $n$
of consecutive regularization updates in constant time to weight $w_j$ is
\begin{equation*}
\begin{split}
w^{(t+n)} & = \sgn(w_j^{(t)}) \left[
		|w_j^{(t)}| \frac{P(t+n-1)}{P(t -1)}  - \lambda_1 P(t+n-1) \left( B(t+n-1) - B(t -1) \right)
	\right]_+
\end{split}
\end{equation*}
\end{proof}


%
%
\section{Lazy Updates for Forward Backward Splitting}
\label{section:lazy-fobos}

Here we turn our attention to FoBoS updates for $\ell^2_2$ and elastic net regularization.
For $\ell_2^2$ regularization, to apply the regularization update
we solve the problem from Equation~(\ref{eqn:fobos-update})
with $\lambda_1$ set to $0$.
%
Solving for $\boldsymbol{w}^*$ gives the update 
$$
w_j^{(t+1)} = \frac{w_j^{(t)}}{1+ \eta^{(t)} \lambda_2}
$$
when $x_{ij} = 0$.
Note that this differs from the standard stochastic gradient descent step.
We can store the values $\Phi(t) = \prod_{\tau=0}^{t} \frac{1}{1+\eta^{t}\lambda_2}$.
Then, the constant time lazy update for FoBoS with $\ell_2^2$ regularization 
to bring a weight current at time $k$ from time $\psi_j$ is
\begin{equation*}
w_j^{(k)} = w_j^{(\psi_j)} \frac{\Phi(k-1)}{\Phi(\psi_j-1)}
\end{equation*}
where $\Phi(t) = (1+\eta^{(t)} \lambda_2)^{-1} \cdot \Phi(t-1)$ 
with base case $\Phi(0) = \frac{1}{1+\eta^{0}\lambda_2}$.

Finally, in the case of elastic net regularization via forward backward splitting, 
we solve the convex optimization problem from Equation~(\ref{eqn:fobos-update}).
This objective also comes apart and can be optimized for each $w_j$ separately. 
Setting the derivative with respect to $w_j$ to zero yields the solution
$$
w_j^{(t+1)} = \sgn(w^{(t)}_j) \left[ \frac{|w_j^{(t)}| - \eta^{(t)}\lambda_1}{\eta^{t} \lambda_2 + 1} \right]_+
$$

\begin{theorem}
A constant-time lazy update for FoBoS with elastic net regularization and decreasing learning rate 
to bring a weight current at time $k$ from time $\psi_j$ is
\begin{equation*}
w_j^{(k)} = \sgn(w_j^{(\psi_j)}) \left[ |w_j^{(\psi_j)}| \frac{\Phi(k-1)}{\Phi(\psi_j-1)}
- \Phi(k-1)\cdot \lambda_1   \left( \beta(k-1)- \beta(\psi_j-1) \right) \right]_+
\end{equation*}
where $\Phi(t)=\Phi(t-1) \cdot \frac{1}{1+\eta^{t} \lambda_2}$ with base case $\Phi(-1) = 1$ 
and $\beta(t) = \beta(t-1) + \frac{\eta^{(t)}}{\Phi(t-1)}$ with base case $\beta(-1) = 0$.
\end{theorem}

\begin{proof}
Write $a_{t} = (\eta^{(t)}\lambda_2+1)^{-1}$ and $b_t=-\eta^{(t)} \lambda_1$. 
Note that neither $a_t$ nor $b_t$ depends upon $w_j$. 
Consider successive updates:
\begin{equation*}
\begin{split}
w_j^{(t+1)}		& = \sgn(w_j^{(t)}) \left[ a_t(|w_j^{(t)}| +b_t) \right]_+ \\
w_j^{(t+n)} 	&= \sgn(w_j^{(t)}) \left[ |w_j^{(t)}| \prod_{\beta=t}^{t+n-1} a_{\beta} +
\sum_{\tau=t}^{t+n-1} \left( b_{\tau} \prod_{\alpha = \tau}^{t+n-1} a_{\alpha} \right) \right]_+.
\end{split}
\end{equation*}
Inside the square brackets, $\frac{\Phi(t+n-1)}{\Phi(t-1)}$ can be substituted for $\prod_{\beta=t}^{t+n-1} a_{\beta}$
and the second term can be expanded as
\begin{equation*}
\begin{split}
\sum_{\tau=t}^{t+n-1} \left( b_{\tau} \prod_{\alpha = \tau}^{t+n-1} a_{\alpha} \right) &=
\sum_{\tau=t}^{t+n-1} \left( b_{\tau} \frac{\Phi(t+n-1)}{\Phi(\tau-1)} \right) \\
		&= - \Phi(t+n-1)\cdot  \lambda_1  \sum_{\tau=t}^{t+n-1} \left( \frac{\eta^{(\tau)}}{\Phi(\tau-1)} \right)   \\
\end{split}
\end{equation*}
Using the dynamic programming approach, for each time $t$, we calculate
$$\beta(t) = \beta(t-1) + \frac{\eta^{(t)}}{\Phi	(t-1)}$$
with the base cases $\beta(0) = \eta^{(0)}$ and $\beta(-1) = 0$. 
Then
\begin{equation*}
w_j^{(t+n)} = \sgn(w_j^{(t)}) \left[ |w_j^{(t)}| \frac{\Phi(t+n-1)}{\Phi(t-1)}
- \Phi(t+n-1)\cdot \lambda_1   \left( \beta(t+n-1)- \beta(t-1) \right) \right]_+
\end{equation*}
\end{proof}

\section{Experiments}
\label{section:experiments}
The usefulness of logistic regression with elastic net regularization is well-known.
To confirm the correctness and speed
of the dynamic programming algorithm just presented, 
we implemented it and tested it on a bag-of-words representation
of abstracts from biomedical articles indexed in Medline as described in \cite{lipton2014optimal}.
The dataset contains exactly $10^6$ examples, $260,941$ features
and an average of $88.54$ nonzero features per document.

We implemented algorithms in Python.
Datasets are represented by standard sparse SciPy matrices.
We implemented both standard and lazy FoBoS
for logistic regression regularized with elastic net.
We confirmed on a synthetic dataset that
the standard FoBoS updates and lazy updates output essentially identical weights.
To make a fair comparison, we also report results where the non-lazy algorithm 
exploits sparsity when calculating predictions.
Even when both methods exploit sparsity to calculate $\hat{y}$,
lazy updates lead to training over $1400$ times faster.
Note that sparse data structures must be used even with dense updates,
because a dense matrix to represent the input dataset would use an unreasonable amount of memory.

Logistic regression with lazy elastic net regularization runs approximately $2000$ times faster 
than with dense regularization updates for both SGD and FoBoS.
In the absence of overhead,
exploiting sparsity should yield a $2947\times$ speedup.
Clearly the additional dynamic programming calculations do not erode the benefits of exploiting sparsity.
While the dynamic programming strategy
consumes space linear in the number of iterations,
it does not present a major time penalty.
Concerning space, storing two floating point numbers for each time step $t$ 
is a modest use of space compared to storing the data itself.
Further, if space ever were a problem, 
all weights could periodically be brought current. 
The cost of this update would be amortized across all iterations and thus would be negligible.

\begin{table}[t]
\centering
\begin{tabular}{ | c | c  | c | c |}
\hline
 & \textbf{Lazy Updates}  & Dense Updates & Dense with Sparse Predictions \\
\hline
SGD  & \textbf{.0102}  & 21.377  &  14.381  \\
\hline
FoBoS &   \textbf{.0120}  & 22.511  &  16.785   \\
\hline
\end{tabular}
\label{table:speed}
\caption{Average time in seconds for each algorithm to process one example.}
\end{table}

\section{Discussion}
\label{section:discussion}
Many interesting datasets are high-dimensional,
and many high-dimensional datasets are sparse.
To be useful, learning algorithms should have time complexity that scales with 
the number of non-zero feature values per example,
as opposed to with the nominal dimensionality.
This paper provides algorithms for fast training of linear models with 
$\ell_2^2$ or with elastic net regularization.
Experiments confirm the correctness and empirical benefit of the method.
In future work we hope to use similar ideas to take advantage of sparsity in nonlinear models,
such as the sparsity provided by rectified linear activation units in modern neural networks.

\subsubsection*{Acknowledgments}
This research was conducted
with generous support from the
Division of Biomedical Informatics
at the University of California, San Diego,
which has funded the first author
via a training grant from
the National Library of Medicine.
Galen Andrew began evaluating lazy updates for multilabel classification
with Charles Elkan in the summer of 2014.
His notes provided an insightful starting point for this research.
Sharad Vikram provided invaluable help in checking
the derivations of closed form updates.

\bibliographystyle{plain}
\bibliography{lazy_updates_arxiv.2}

\end{document}